\documentclass[10pt]{article}

\usepackage[utf8]{inputenc}
\usepackage[T1]{fontenc}

\usepackage{epsf}
\usepackage{amsmath}

\allowdisplaybreaks

\usepackage[showframe=false]{geometry}
\usepackage{changepage}

\usepackage{epsfig}
\usepackage{amssymb}

\usepackage{amsthm}
\usepackage{setspace}
\usepackage{cite}
\usepackage{mcite}

\usepackage{algorithmic}  
\usepackage{algorithm}

\usepackage{shadow}
\usepackage{fancybox}
\usepackage{fancyhdr}

\usepackage{color}
\usepackage[usenames,dvipsnames,svgnames,table]{xcolor}
\newcommand{\bl}[1]{\textcolor{blue}{#1}}

\definecolor{mypurple}{rgb}{.4,.0,.5}

\usepackage[hyphens]{url}

\usepackage[colorlinks=true,
            linkcolor=black,
            urlcolor=blue,
            citecolor=purple]{hyperref}

\usepackage{breakurl}

\def\y{{\bf y}}

\def\x{{\bf x}}

\def\x{{\mathbf x}}

\def\u{{\bf u}}

\def\x{{\bf x}}
\def\y{{\bf y}}
\def\z{{\bf z}}
\def\q{{\bf q}}
\def\m{{\bf m}}

\def\b{{\bf b}}
\def\c{{\bf c}}

\def\h{{\bf h}}

\def\cH{{\mathcal H}}

\def\be{\begin{equation}}
\def\ee{\end{equation}}
\def\ba{\left[\begin{array}}
\def\ea{\end{array}\right]}

\def\u{{\bf u}}

\def\x{{\bf x}}
\def\y{{\bf y}}
\def\z{{\bf z}}
\def\q{{\bf q}}

\def\b{{\bf b}}
\def\c{{\bf c}}

\def\p{{\bf p}}

\def\1{{\bf 1}}

\def\0{{\bf 0}}

\def\erfc{\mbox{erfc}}

\def\mR{{\mathbb R}}
\def\mN{{\mathbb N}}
\def\mE{{\mathbb E}}
\def\mB{{\mathbb B}}
\def\mS{{\mathbb S}}
\def\mP{{\mathbb P}}

\def\lp{\left (}
\def\rp{\right )}

\def\y{{\bf y}}

\def\x{{\bf x}}

\def\x{{\mathbf x}}

\def\u{{\bf u}}

\def\x{{\bf x}}
\def\y{{\bf y}}
\def\z{{\bf z}}
\def\q{{\bf q}}

\def\b{{\bf b}}
\def\c{{\bf c}}

\def\h{{\bf h}}

\def\cH{{\cal H}}

\def\be{\begin{equation}}
\def\ee{\end{equation}}
\def\ba{\left[\begin{array}}
\def\ea{\end{array}\right]}

\def\u{{\bf u}}

\def\x{{\bf x}}
\def\y{{\bf y}}
\def\z{{\bf z}}
\def\q{{\bf q}}

\def\b{{\bf b}}
\def\c{{\bf c}}

\def\p{{\bf p}}

\def\({\left (}
\def\){\right )}

\def\1{{\bf 1}}
\def\m{{\bf m}}
\def\q{{\bf q}}

\def\0{{\bf 0}}

\def\cX{{\mathcal X}}
\def\cY{{\mathcal Y}}

\usepackage{xcolor}
\usepackage{color}

\definecolor{darkgreen}{rgb}{0, 0.4,0}

\definecolor{purplebrown}{rgb}{0.5,0.1,0.6}

\definecolor{ultclupcol}{rgb}{0.1,0.5,0.5}

\definecolor{mytrycolor}{rgb}{0.5,0.7,0.2}

\definecolor{ultclupcola}{rgb}{.5,0,.5}

\definecolor{shadebrown}{rgb}{0.1,0.1,0.9}
\definecolor{lightblue}{rgb}{0.2,0,1}

\usepackage{fancybox}
\usepackage{graphicx}
\usepackage{epstopdf}
\usepackage{epsfig}
\usepackage{wrapfig}
\usepackage{subfigure}

\usepackage{xcolor}
\usepackage{tcolorbox}

\newtcbox{\xmybox}{on line,
arc=7pt,
before upper={\rule[-3pt]{0pt}{10pt}},boxrule=0pt,
boxsep=0pt,left=6pt,right=6pt,top=0pt,bottom=0pt,enhanced, coltext=blue, colback=white!10!yellow}

\newtcbox{\xmyboxa}{on line,
arc=7pt,
before upper={\rule[-3pt]{0pt}{10pt}},boxrule=0pt,
boxsep=0pt,left=6pt,right=6pt,top=0pt,bottom=0pt,enhanced, colback=white!10!yellow}

\newtcbox{\xmyboxb}{on line,
arc=7pt,
before upper={\rule[-3pt]{0pt}{10pt}},boxrule=1pt,colframe=darkgreen!100!blue,
boxsep=0pt,left=6pt,right=6pt,top=0pt,bottom=0pt,enhanced, colback=white!10!yellow}

\newtcbox{\xmyboxc}{on line,
arc=7pt,
before upper={\rule[-3pt]{0pt}{10pt}},boxrule=.7pt,colframe=blue!100!blue,
boxsep=0pt,left=6pt,right=6pt,top=0pt,bottom=0pt,enhanced, coltext=blue, colback=white!10!yellow}

\newtcbox{\xmytboxa}{on line,
arc=7pt,
before upper={\rule[-3pt]{0pt}{10pt}},boxrule=.0pt,colframe=pink!50!yellow,
boxsep=0pt,left=6pt,right=6pt,top=0pt,bottom=0pt,enhanced, coltext=white, colback=blue!40!red}

\newtcbox{\xmytboxb}{on line,
arc=7pt,
before upper={\rule[-3pt]{0pt}{10pt}},boxrule=.0pt,colframe=pink!50!yellow,
boxsep=0pt,left=6pt,right=6pt,top=0pt,bottom=0pt,enhanced, coltext=white, colback=white!40!green}

\makeatletter
\newcommand\subsubsubsection{\@startsection{paragraph}{4}{\z@}{-2.5ex\@plus -1ex \@minus -.25ex}{1.25ex \@plus .25ex}{\normalfont\normalsize\bfseries}}
\newcommand\subsubsubsubsection{\@startsection{subparagraph}{5}{\z@}{-2.5ex\@plus -1ex \@minus -.25ex}{1.25ex \@plus .25ex}{\normalfont\normalsize\bfseries}}
\makeatother

\newtheorem{theorem}{Theorem}

\begin{document}

\begin{singlespace}

\title {Rare dense solutions clusters in asymmetric binary perceptrons -- local entropy via fully lifted RDT 
}
\author{
\textsc{Mihailo Stojnic
\footnote{e-mail: {\tt flatoyer@gmail.com}} }}
\date{}
\maketitle

\centerline{{\bf Abstract}} \vspace*{0.1in}

We study classical asymmetric binary perceptron (ABP) and associated \emph{local entropy} (LE) as potential source of its algorithmic hardness. Isolation of \emph{typical} ABP solutions in SAT phase seemingly suggests a universal algorithmic hardness. Paradoxically, efficient algorithms do exist even for constraint densities $\alpha$ fairly close but at a finite distance (\emph{computational gap}) from the capacity. In recent years, existence of rare large dense clusters and magical ability of fast algorithms to find them have been posited as the conceptual resolution of this paradox. Monotonicity or breakdown of the LEs associated with such \emph{atypical} clusters are predicated to play a key role in their thinning-out  or even complete defragmentation.

Invention of fully lifted random duality theory (fl RDT) \cite{Stojnicnflgscompyx23,Stojnicsflgscompyx23,Stojnicflrdt23} allows studying random structures \emph{typical} features. A large deviation upgrade, sfl LD RDT  \cite{Stojnicnflldp25,Stojnicsflldp25},  moves things further and enables \emph{atypical} features characterizations as well. Utilizing the machinery of \cite{Stojnicnflldp25,Stojnicsflldp25} we here develop a generic framework to study LE as an ABP's atypical feature. Already on the second level of lifting we discover that the LE results are closely matching those obtained through replica methods. For classical zero threshold ABP, we obtain that LE breaks down for $\alpha$ in $(0.77,0.78)$ interval which basically matches $\alpha\sim 0.75-0.77$ range that currently best ABP solvers can handle and effectively indicates that LE's behavior  might indeed be among key reflections of the ABP's computational gaps presumable existence.

\vspace*{0.25in} \noindent {\bf Index Terms: Binary perceptrons; Local entropy; Random duality theory; Large deviations}.

\end{singlespace}

\section{Introduction}
\label{sec:back}

We consider  \emph{perceptrons} -- the most fundamental neural networks (NN) units and unavoidable machine learning (ML) classifying tools.  As basically irreplaceable AI concepts they have been extensively studied over the last 70 years. Such studies include both looking at them as foundational parts of more complex AI architectures or as simple convenient individual models presumably indicative of more generic machine learning phenomenologies. Two classical perceptron types, \emph{spherical} and \emph{binary},  distinguished themselves from the early machine learning days and it is by no means an overstretching  to say that determining spherical perceptron's (zero threshold) capacity (critical data constraints density $\alpha_c$ below which classifying is successful)  is among the monumental AI breakthroughs \cite{Wendel62,Winder,Cover65}. In particular, as one of the very first  analytical confirmations of AI's practical potential, it crucially helped understand importance of deep analytical foundations. Moreover, continuation of such a tendency through the ensuing decades effectively ensured that modern AI is almost unimaginable without a strong theoretical support.

After initial success of  \cite{Wendel62,Winder,Cover65},  a large body of highly-sophisticated  followup work
\cite{Gar88,GarDer88,SchTir02,SchTir03,Tal05,Talbook11a,Talbook11b,StojnicGardGen13,StojnicGardSphNeg13,StojnicGardSphErr13} significantly  extended understanding of perceptrons often allowing shifting focus on features way beyond the classical capacities. The class of \emph{positive} spherical perceptrons (PSP) (of which the above mentioned zero-threshold one is a special case) became particularly well  understood. Strong deterministic duality and convexity as a PSP's distinguished underlying features were recognized as main catalyst for success of analytical studies \cite{SchTir02,SchTir03,StojnicGardGen13}. At the same time, lack of these features was observed as often analytically unsurpassable obstacle. This is clearly exemplified by the negative spherical perceptron (NSP) where a seemingly tiny change from a positive to a negative threshold makes things much harder \cite{StojnicGardSphNeg13,FPSUZ17,FraHwaUrb19,FraPar16,FraSclUrb19,FraSclUrb20,AlaSel20,BaldMPZ23,BMPZ23} and (compared to  \cite{SchTir02,SchTir03,StojnicGardGen13,Wendel62,Winder,Cover65}) substantially more involved  approaches are needed   \cite{Stojnicsflgscompyx23,Stojnicnflgscompyx23,Stojnicflrdt23,Stojnicnegsphflrdt23}.

The asymmetric binary perceptrons (ABP) that we study here are somewhat similar to NSPs as underlying strong deterministic duality is lacking. For example, simple replica symmetric predictions do not hold \cite{Gar88,GarDer88,StojnicDiscPercp13} and instead, more involved,  replica symmetry breaking ones from \cite{KraMez89} are needed to obtain accurate capacity characterizations \cite{DingSun19,NakSun23,BoltNakSunXu22,Huang24,Stojnicbinperflrdt23} (somewhere in between the PSP and the ABP are the symmetric binary perceptrons (SBP) where replica symmetric predictions also do not hold but a convenient combinatorial nature of the problem still allows for simple precise characterizations \cite{AbbLiSly21b,PerkXu21,AbbLiSly21a,AubPerZde19,GamKizPerXu22}).

\section{ABP algorithmic hardness}
\label{sec:examples}

All of the above relates to perceptrons' theoretical limits. The first next question is if these limits are practically attainable.  Classical complexity theory positions ABP's algorithmic solving as an NP problem. Due to its  \emph{worst case} nature, NP-ness rarely gives a proper explanation regarding typical solvability.  Practically speaking, things are actually even worse as many excellent algorithms that perform very well in a large part of $\alpha< \alpha_c$ interval actually exist \cite{BrZech06,BaldassiBBZ07,Hubara16,KimRoc98}. To put everything in concrete terms, the capacity of zero-threshold ABP is $\alpha_c\approx 0.833$. Known fast algorithms can solve ABP up to  $\alpha_{alg}\approx 0.75$ (sometimes even up to $0.77$). Determining precisely $\alpha_{alg}$ -- the critical constraint density  below which the efficient algorithms exist -- is an extraordinary challenge. Characterizing ABP's underlying structural phenomenology that actually allows for the existence of such algorithms is likely to be an equally challenging associated task. Since the current $\alpha_{alg}$ is below $\alpha_c$ one presumes the existence of the so-called \emph{computational gap} (C-gap) -- a key algorithmic feature already associated with a host of other (random) optimization/satisfiability problems \cite{MMZ05,GamarSud14,GamarSud17,GamarSud17a,AchlioptasR06,AchlioptasCR11,GamMZ22}.

\subsection{ABP C-gap demystification  -- related prior work}
\label{sec:examples}

Over the last 20 years, a strong effort has been put forth to improve understanding of the above concepts. Whether or not the C-gap indeed exists and determining its precise value if it does exist are among key open questions. Similar questions have been studied in classes of optimization problems extending way beyond the ABP. While many excellent results have been obtained, any form of a generic resolution and a possible overall C-gap demystification remain out of reach. To give a bit of a flavor as to what kind of approaches are possible, we single out possibly the most attractive among them that focuses on studying the nature of clustered solutions. Two clustering related properties have received particular attention in recent literature: (i) \emph{Overlap gap property} (OGP) (see, e.g., \cite{Gamar21,GamarSud14,GamarSud17,GamarSud17a,AchlioptasCR11,HMMZ08,MMZ05}); and (ii) \emph{Local entropy} (LE) (see, e.g., \cite{Bald15,Bald16,Bald20}).

The OGP approach \cite{Gamar21,GamarSud14,GamarSud17,GamarSud17a,AchlioptasCR11,HMMZ08,MMZ05} attributes algorithmic efficacy to a lack of gaps in the spectrum of attainable  solutions pairs (or generally $m$-tuples) Hamming distances. For analytically simpler but structurally presumably similar SBP alternative, it is known  \cite{GamKizPerXu22,Bald20} that the OGP's presence extends well below $\alpha_c$ (see also \cite{GamKizPerXu23} for similar discrepancy minimization related conclusions). Assuming that OGP indeed impacts algorithmic tractability, this provides a rather strong indication of C-gap existence. On the other hand, a recent shortest path counterexample \cite{LiSch24} disproves OGP generic hardness implication (earlier disproving examples included $k$-XOR ones and, due to their algebraic simplicity, have been taken as exceptions). One should however note that  \cite{LiSch24} does not disprove OGP's relevance for particular algorithms or  different optimization problems. For example, long believed OGP absence in the famous (binary) Sherrington-Kirkpatrick (SK) model \cite{SheKir72}  directly implies their solvability by polynomial algorithms \cite{Montanari19} (related p-spin SK extensions are obtained as well \cite{AlaouiMS22,AlaouiMS21}; for earlier spherical spin-glass models related considerations see \cite{Subag17,Subag17a,Subag21,Subag24} and  for importance of more sophisticated OGPs that go beyond the standard ones related to solution pairs see, e.g., \cite{Kiz23,HuangS22}). While the OGP role in generic algorithmic hardness remains undetermined, its presence certainly precludes existence of efficient implementations for many specific  algorithmic classes \cite{GamKizPerXu22}. Moreover, existence of practical algorithms in $\alpha$ ranges where OGP is absent additionally strengthens its relevance for many well known problems \cite{RahVir17,GamarSud14,GamarSud17,GamAW24,Wein22}.

Alternatively to the above OGP view, \cite{Huang13,Huang14}  considered \emph{typical} solutions entropy as a way of describing clustering organization and its relevance for algorithmic hardness. A complete so-called \emph{frozen} isolation of typical solutions is predicated (and proven in \cite{PerkXu21,AbbLiSly21a,AbbLiSly21b} for SBP). \cite{Bald15,Bald16,Bald20} went further and considered a stronger entropic refinement. They looked at \emph{atypical} well connected clusters and proposed studying associated \emph{local entropy} (LE). The main idea is that even though predominant typical solutions might be disconnected from each other and unreachable via local searches \cite{Huang13,Huang14,PerkXu21,AbbLiSly21b}, there may still exist rare well-connected clusters. It is then predicted that efficient algorithms find precisely such rare clusters (see, e.g., \cite{ElAlGam24} for an SBP's sampling type of justification along these lines). If such a pictorial portrayal is indeed true, then C-gap existence is likely directly related to the properties of rare clusters. Behavior of the associated local entropy (monotonicity, breakdown, negativity, etc.) is further speculated in \cite{Bald15,Bald16,Bald20} as a key reflection of rare clusters' structures and its impact on algorithmic hardness. In support of such a phenomenology, \cite{AbbLiSly21a} showed for SBP the existence of maximal diameter clusters for sufficiently small $\alpha$. Moreover, \cite{AbbLiSly21a} also showed that similar clusters, albeit of linear diameter, exist even for any $\alpha<\alpha_c$ (modulo tiny technical assumptions, \cite{AbbLiSly21a}'s SBP results extend to ABP as well). Reconnecting back with the OGP, \cite{BarbAKZ23} showed that small $\alpha$ SBP LE results scaling-wise fairly closely match the \cite{GamKizPerXu22}'s OGP predictions  (and modulo a log term the \cite{BanSpen20}'s algorithmic performance). While such a nice OGP -- LE correspondence is rather convincing, establishing a definite answer as to whether or not these properties indeed crucially impact appearance of C-gaps remains a mystery. At the same time, it should be pointed out that no matter what their C-gap relevance is, these properties certainly  provide deep insights into the intrinsic nature of random structures and understanding them is of independent interest as well.

\subsection{ABP C-gap demystification  -- our contributions}
\label{sec:contrib}

More foundational nature of the above concepts makes them analytically much harder when compared to studying ABP's capacity alone. As expected, the source of hardness is related to the move from studying \emph{typical} to studying \emph{atypical} random structures features. Such a leap is rather gigantic and, with a few exceptions, the analytical progress is often limited to statistical mechanics replica methods.

For studying \emph{typical} features, the power of the mechanisms introduced in \cite{Stojnicnflgscompyx23,Stojnicsflgscompyx23,Stojnicflrdt23} suffices. Clearly, the most prominent ABP typical feature is its critical capacity $\alpha_c$  and the machinery of \cite{Stojnicnflgscompyx23,Stojnicsflgscompyx23,Stojnicflrdt23} can be used to determine it\cite{Stojnicbinperflrdt23}  ($\alpha_c$ is the critical constraint density value for which ABP associative memory functioning transitions from  successful to unsuccessful phase). The same mechanisms can also be used to characterize various other ABP features associated with \emph{typical} behavior. For example, in regimes below the capacity ($\alpha< \alpha_c$), a large number of hierarchically ordered solutions (ABP memories) will be present. The \emph{typical} solutions interrelations and entropies   can be precisely determined via \cite{Stojnicnflgscompyx23,Stojnicsflgscompyx23,Stojnicflrdt23}. On the other hand, as we will see below, handling \emph{atypical} features finds an fl RDT upgrade rather useful.

Following into the footsteps of \cite{Stojnicbinperflrdt23} we here move the \emph{typical} connection between the statistical ABP and fully lifted random duality theory (fl RDT) \cite{StojnicCSetam09,StojnicICASSP10var,StojnicRegRndDlt10,StojnicGardGen13,StojnicICASSP09}  to a large deviation \emph{atypical} level. Leveraging the fl RDT's large deviation sfl LD RDT upgrade from \cite{Stojnicnflldp25,Stojnicsflldp25}, we here develop a generic framework for studying ABP's LE. Similarly to fl RDT, to have sfl LD RDT fully operational a series of underlying numerical evaluations is needed. After conducting those for zero-threshold ($\kappa=0$) ABP we uncover that already on the second level of lifting (2-sfl LD RDT), the associated worst case LE exhibits a breakdown in $(0.77,0.78)$ $\alpha$ interval. This matches both the replica predictions of \cite{Bald16} and the range  $\alpha\sim 0.75-0.77$ that best ABP solvers cam handle which ultimately indicates that the LEs indeed may be among the key ABP C-gaps presumable existence reflections.

\subsubsection{Beyond ABP}
\label{sec:contribbey}

As stated earlier, the ABP that we study here is a type of general perceptron concept that has dominated machine learning and neureal networks studies over the last several decades. However, the methodology that we present extends way beyond perceptrons. It is easily applicable to
many closely connected feasibility problems  including SBPs  \cite{AubPerZde19,AbbLiSly21a,AbbLiSly21b,Bald20,GamKizPerXu22,PerkXu21,ElAlGam24,SahSaw23,Barb24,djalt22,BarbAKZ23},
PSPs and NSPs \cite{StojnicGardGen13,StojnicGardSphErr13,StojnicGardSphNeg13,GarDer88,Gar88,Schlafli,Cover65,Winder,Winder61,Wendel,Cameron60,Joseph60,BalVen87,Ven86,SchTir02,SchTir03}, and compressed sensing  so-called $\ell_1$  \emph{sectional/strong} phase transitions \cite{StojnicCSetam09,StojnicLiftStrSec13}, to name a few. In addition to extension to different types of perceptrons (or single-unit nets), the extensions to  multi-unit multi-layer (deep) nets are possible as well. Excellent results obtained in this direction via replica methods  \cite{BalMalZech19} indicate that clustering algorithmic effects do translate as the network architectures get more complex. Besides feasibility problems, without much of additional conceptual effort, our results also extend to  many optimal objective seeking standard optimizations including discrepancy minimization \cite{KaKLO86,Spen85,LovMek15,GamKizPerXu23,Roth17,AlwLiuSaw21} and a host of others discussed in \cite{Stojnicnflgscompyx23,Stojnicsflgscompyx23,Stojnicflrdt23,Stojnichopflrdt23}. For example, studies of famous Hopfield  and closely related Little models as well as more mathematically oriented restricted isometries \cite{Hop82,PasFig78,Hebb49,PasShchTir94,ShchTir93,BarGenGueTan10,BarGenGueTan12,Tal98,StojnicMoreSophHopBnds10,BovGay98,Zhao11,Talbook11a,Talbook11b,Stojnicnflgscompyx23,Stojnichopflrdt23} focus on the associated free energies that exhibit behavior similar to the perceptron's one discussed above. The ground state energy values are typically achieved by a single optimal configuration which is analogous to having a \emph{typical} solution associated with ABP's $\alpha_c$. On the other hand, energetic (band) levels  away from the ground state can be achieved either by a large set or by no configurations at all. This is again fully analogous to ABP, where moving below $\alpha_c$ (to the SAT phase) allows for (exponentially) many solutions and moving above $\alpha_c$ (to the UNSAT phase) disallows existence of any satisfiability configuration. In other words, even though these problems are unrelated to satisfiability ones per se, they do exhibit similar phase transitioning behavior. Moreover, in all these scenarios, both \emph{typical} and \emph{atypical} behavior can be studied via the framework introduced here.

\section{Mathematical setup and local entropy}
 \label{sec:bprfps}

Following into the footsteps of \cite{StojnicDiscPercp13,StojnicGardGen13,Stojnicbinperflrdt23,Stojnicnegsphflrdt23,GarDer88,Gar88,StojnicGardSphNeg13}, we start by recognizing that perceptrons, as key units of almost any modern machine learning concept, belong to a general class of feasibility problems
\begin{eqnarray}
\hspace{-1.5in}\mbox{$\mathbf{\mathcal F}(G,\b,\cX,\alpha)$:} \hspace{1in}\mbox{find} & & \x\nonumber \\
\mbox{subject to}
& & G\x\geq \b \nonumber \\
& & \x\in\cX, \label{eq:ex1}
\end{eqnarray}
with $G\in\mR^{n\times n}$, $\b\in\mR^{m\times 1}$, and $\cX\in\mR^n$ (large \emph{linear} regime is considered where $m$ and $n$ are such that $\alpha= \lim_{n\rightarrow\infty} \frac{m}{n}$ remains constant). The above mentioned two most prominent perceptrons' types, \emph{spherical} and \emph{binary}  are obtained as special case of (\ref{eq:ex1}). Specialization to $\cX=\{\x | \| \x \|_2=1\} \triangleq \mS^m$ with $\b\geq 0$ gives the positive (or standard) spherical perceptron  (PSP) \cite{StojnicGardGen13,StojnicGardSphErr13,GarDer88,Gar88,Schlafli,Cover65,Winder,Winder61,Wendel62,Cameron60,Joseph60,BalVen87,Ven86,SchTir02,SchTir03}. The same specialization with $\b< 0$ gives the negative spherical perceptron (NSP) counterpart \cite{FPSUZ17,Talbook11a,FraHwaUrb19,FraPar16,FraSclUrb19,FraSclUrb20,AlaSel20,StojnicGardSphNeg13,Stojnicnegsphflrdt23,BaldMPZ23}. On the other hand,
specialization to $\cX=\{-\frac{1}{\sqrt{n}},\frac{1}{\sqrt{n}} \}^n \triangleq \mB^n$ gives standard or asymmetric binary perceptron (ABP) \cite{Talbook11a,StojnicGardGen13,GarDer88,Gar88,StojnicDiscPercp13,KraMez89,GutSte90,KimRoc98,NakSun23,BoltNakSunXu22,PerkXu21,CXu21,DingSun19,Huang24,Stojnicbinperflrdt23,LiSZ24} and an additional tiny change of the linear constraints, $|G\x |\leq \b$, gives its symmetric binary perceptron (SBP) alternative \cite{AubPerZde19,AbbLiSly21a,AbbLiSly21b,Bald20,GamKizPerXu22,PerkXu21,ElAlGam24,SahSaw23,Barb24,djalt22,BarbAKZ23} (for a closely related discrepancy minimization problems, see, e.g., \cite{KaKLO86,Spen85,LovMek15,GamKizPerXu23,Roth17,AlwLiuSaw21}). To make writing easier, we consider ABP with \emph{fixed} threshold where $\b=\kappa\1$, $\kappa\in\mR$, and $\1$ is a column vector of all ones. Statistical context with $G$ having independent standard normal elements is the main focus throughout the paper.

After setting
\begin{eqnarray}
\xi_{ABP}
& =  &
 \min_{\x\in \mB^n} \max_{\y\in\mS_+^m}  \lp -\y^TG\x + \kappa \y^T\1 \rp,
 \label{eq:ex3}
\end{eqnarray}
with $\mB^n$ being the vertices of the $n$-dimensional unit cube and $\mS_+^m$ being the $m$-dimensional unit sphere  positive orthant part (i.e., $\mS_+^m=\{\y|\|\y\|_2=1,\y\geq 0\}$),
 \cite{StojnicGardGen13,StojnicGardSphErr13,StojnicGardSphNeg13,StojnicDiscPercp13,Stojnicbinperflrdt23} recognized  ABP's  \emph{capacity} as
 \begin{eqnarray}
\alpha & = &    \lim_{n\rightarrow \infty} \frac{m}{n}  \nonumber \\
\alpha_c(\kappa) & \triangleq & \max \{\alpha |\hspace{.08in}  \lim_{n\rightarrow\infty}\mP_G\lp  \xi_{ABP}>0\rp\longrightarrow 1\} \nonumber \\
& = & \max \{\alpha |\hspace{.08in}  \lim_{n\rightarrow\infty}\mP_G\lp{\mathcal F}(G,\b,\cX,\alpha) \hspace{.07in}\mbox{is feasible} \rp\longrightarrow 1\}.
  \label{eq:ex4}
\end{eqnarray}
As usual, throughout the paper we adopt the convention that subscripts next to $\mP$ or $\mE$ denote the randomness for which the statistical evaluations are done (when this is obvious from the context, specifying the subscripts is skipped).

The capacity reflects the so-called phase-transitioning nature of the underlying random ABP structure. In particular, for $\alpha>\alpha_c$ the problem is infeasible or in the satisfiability  terminology in the so-called UNSAT phase.  On the other hand, for $\alpha<\alpha_c$ it is feasible and in the SAT phase. The transition between phases happens exactly at $\alpha_c$ where the exponentially large volume of admissible solutions  shrinks so that eventually no single point is admissible \cite{KraMez89,NakSun23,BoltNakSunXu22,DingSun19,Huang24,Stojnicbinperflrdt23}. Such a picture corresponds to ABP's theoretical limit. It is manifested through the so-called \emph{typical} behavior which by now is very well understood. On the other hand, algorithmically achieving such a limit presents a serious challenge and no known fast algorithm (say, polynomial)  provably solves ABP for $\alpha\approx \alpha_c$. However, for sufficiently small $\alpha$ they provably do exist. Moreover, many practical implementations exist even for $\alpha$ not that far away from $\alpha_c$ (for example, for zero-threshold ($\kappa=0$) ABP practically efficient algorithms can be designed even for $\alpha\sim 0.75-0.77$ which indeed is not that far away from $\alpha_c\approx 0.833$).  So-called \emph{atypical} features are believed  to be responsible for such a picture of  ABP algorithmic hardness and presumable existence of C-gaps (inability to achieve theoretical performance with fast algorithms). Namely, while the typical solutions for any $\alpha<\alpha_c$ might be far away (disconnected) thereby portraying the problem as seemingly hard, the existence of rare (sub-dominant) clusters of solutions might still allow for design of fast algorithms. This is precisely what \cite{Bald15,Bald16} predicted while studying the local entropy of such rare clusters via replica methods. In particular, monotonicity and/or breakdown are the key local entropy features established in \cite{Bald15,Bald16} (as mentioned earlier, in support of \cite{Bald15,Bald16} predictions, modulo tiny technical assumptions \cite{AbbLiSly21a} showed that for sufficiently small $\alpha$ there is an exponentially large cluster of \emph{maximal} diameter; furthermore, existence of such clusters but of \emph{linear} diameter was also shown in \cite{AbbLiSly21a} even for any $\alpha<\alpha_c$).

The sfl RDT machinery introduced in \cite{Stojnicnflgscompyx23,Stojnicsflgscompyx23} can handle studying typical behaviors. To study local entropy as an atypical ABP feature an upgrade of sfl RDT \cite{Stojnicnflgscompyx23,Stojnicsflgscompyx23} might be useful. Companion papers \cite{Stojnicnflldp25,Stojnicsflldp25} provide such an upgrade and  we show below how it can be utilized. Before we get to concretely implement mechanisms of\cite{Stojnicnflldp25,Stojnicsflldp25}, we need several technical preliminaries that we discuss next.

\subsection{ABP local entropy and associated free energy}
\label{secrfpsfe}

The above mentioned ABP's local entropy (LE) is a measure of its solutions density. Various LEs can be defined and they predominantly  depend on the type of the so-called reference solution (configuration) and the associated cost functions (for more on the utilization of different LE forms, see, e.g., \cite{Bald15,Bald16,Bald21,BaldMPZ23,BMPZ23}). The most typical LE scenarios rely on running over all satisfiability configurations or weighing them via certain distributions as in the so-called Franz-Parisi potential \cite{FraPar95}. For example, one can choose the reference configuration more or less uniformly among all solutions and then evaluate (scaled by $n$) $\log$ of the size of the associated cluster's part at a given Hamming distance $d=\frac{1-\bar{\delta}}{2}$ with $\bar{\delta}$ being the configurational overlap (for more on the role of different loss functions in reference sampling distributions and the so-called high/low margins $\kappa$ see, e.g., \cite{BMPZ23,BaldMPZ23,Barb24,BarbAKZ23}). The key prediction of \cite{Bald15} is that if the cluster sizes/entropies are monotonically increasing with $d$, then solutions remain well connected through simple local fluctuations and are likely achievable by algorithms that promote such searches. A so to say \emph{worst case} LE alternative (introduced in \cite{Bald16}) assumes everything as above except that the reference configuration  itself need not be a solution. Two observations are in place: \textbf{\emph{(i)}} even if there is no insisting that the reference configuration is a solution, it is highly likely that it will be (simply as being a center of an exponentially large cluster of solutions and particularly so if the cluster's diameter is small); and \textbf{\emph{(ii)}} if  for some $\alpha$ there is an interval of distances from each reference point without exponentially many solutions  then the same holds for reference configurations that are themselves solutions. The second point effectively ensures the LE's worst case nature. In particular, one views the worst case LE  breakdown as an implication of a the lack of dense (``exponentially large'') local connectivity.

We formally write ABP's LE as
\begin{equation}
S_l(\bar{\delta}) = \lim_{n,\beta,\beta_z,p\rightarrow\infty} \mE_G \frac{  \log \lp   \sum_{\bar{\x}  \in \mB^n  } \lp \sum_{\x\in\mB^n, \bar{\x}^T \x=\bar{\delta}  }  e^{ \beta \beta_z (\1^Th(G\x-\kappa) - m  )  }  \rp^p   \rp      }  {n p},\label{eq:le1}
\end{equation}
where standard Heaviside function $h(\cdot)$ acts componentwise on its vector arguments. The above LE basically does the following: \textbf{\emph{(i)}} first, the term  in the inner
parenthesis counts the number of solutions at Hamming distance $d=\frac{1-\bar{\delta}}{2}$ from a selected reference configuration $\bar{\x}$ (to be at Hamming distance $d$ from $\bar{\x}$ alternatively means to have overlap $\bar{\delta}$ with $\bar{\x}$); and \textbf{\emph{(ii)}} second, the outer summation then selects the reference configuration that has the largest local entropy at distance $\bar{\delta}$. One then recognizes that (\ref{eq:le1})
can be rewritten as
\begin{eqnarray}\label{eq:le2}
S_l(\bar{\delta})&  = &
\lim_{n,\beta,\beta_z,p\rightarrow\infty} \mE_G\frac{ \log \lp   \sum_{\bar{\x}  \in \mB^n  } \lp \sum_{\x\in \mB^n, \bar{\x}^T \x=\bar{\delta}  }  e^{ \beta \max_{\z} \min_{\y\in\mS^m }  \lp  \beta_z (\1^T h(\z-\kappa) - m  )  + \y^T (G\x -\z) \rp  }  \rp^p   \rp      }  {n p}
 \nonumber \\
 & = &
  \lim_{n,\beta,\beta_z,p\rightarrow\infty}\mE_G \frac{ \log \lp   \sum_{\bar{\x}  \in \mB^n  } \lp \sum_{\x\in\mB^n, \bar{\x}^T \x=\bar{\delta} } \max_{\z}  e^{ - \beta \max_{\y\in\mS^m }  \lp  \beta_z (\1^Th(\z-\kappa) - m  )  - \y^T (G\x -\z) \rp  }  \rp^p   \rp      }  {n p}
 \nonumber \\
& = &
  \lim_{n,\beta,\beta_z,p\rightarrow\infty}\mE_G \frac{ \log \lp   \sum_{\bar{\x}  \in \mB^n  } \lp
  \sum_{\x\in\mB^n, \bar{\x}^T \x=\bar{\delta} }
  \sum_{\z}
  \lp
  \sum_{\y\in\mS^m}  e^{ \beta   \lp  \beta_z (\1^Th(\z-\kappa) - m  )  - \y^T (G\x -\z) \rp  }
    \rp^{-1}
     \rp^p    \rp    }  {n p}
 \nonumber \\
 & = &
  \lim_{n,\beta,\beta_z,p\rightarrow\infty} \mE_G \frac{\log \lp   \sum_{\bar{\x}  \in \mB^n  } \lp \sum_{\x\in\mB^n, \bar{\x}^T \x=\bar{\delta} ,\z }
  \lp
    \sum_{\y\in\mS^m}  e^{ \beta   \lp  \beta_z (\1^Th(\z-\kappa) - m  )  - \y^T (G\x -\z) \rp  }
    \rp^{-1}
     \rp^p   \rp    }  {n p}.
\end{eqnarray}
We now look at the corresponding particular type of free energy function. Consider the following Hamiltonian
\begin{equation}
\cH(G)= -\y^TG\x,\label{eq:ham1}
\end{equation}
and  partition function
\begin{equation}
Z_{\bar{\x}}(\beta,G)=\sum_{\x\in\mB^n, \bar{\x}^T \x=\bar{\delta},\z } \lp \sum_{\y\in\mS^m}e^{\beta \lp f^0(\z)  +   \cH(G)   \rp   }    \rp^{-1},  \label{eq:partfun}
\end{equation}
In \cite{Stojnicbinperflrdt23} thermodynamic limit average ``\emph{reciprocal}'' free energy was associated with (a slightly simpler) $Z_{\bar{\x}}(\beta,G)$. Here we associate the following large deviation type of reweighted variant
\begin{eqnarray}
f_{sq}(\beta) & = &
\lim_{n,p\rightarrow\infty}\frac{\mE_G\log \lp \sum_{\bar{\x}\in\mB^n}  \lp Z_{\bar{\x}}(\beta,G)  \rp^p    \rp   }{np}
\nonumber \\
& = &
\lim_{n,p\rightarrow\infty}\frac{\mE_G\log \lp \sum_{\bar{\x}\in \mB^n }  \lp
\sum_{\x\in\mB^n, \bar{\x}^T \x=\bar{\delta} ,\z } \lp \sum_{\y\in\mS^m  }  e^{\beta \lp f^0(\z)  +   \cH(G)   \rp   }    \rp^{-1}
 \rp^p   \rp  }{np}
\nonumber \\
& = &
\lim_{n,p\rightarrow\infty}\frac{\mE_G\log \lp \sum_{\bar{\x}\in\mB^n }  \lp
\sum_{ \x\in\mB^n, \bar{\x}^T \x=\bar{\delta} ,\z  } \lp \sum_{\y\in\mS^m  }e^{\beta \lp f^0 (\z)  -   \y^T G\x   \rp   }    \rp^{-1}
 \rp^p   \rp  }{np}.\label{eq:logpartfunsqrt}
\end{eqnarray}
Choosing $f^0(\z) = \beta_z(\1^Th(\z-\kappa) - m ) +\y^T\z $, one has for the ground state
\begin{eqnarray}
f_{sq}(\infty)   \triangleq    \lim_{\beta,\beta_z\rightarrow\infty}f_{sq}(\beta)  =
 S_l(\bar{\delta}).  \label{eq:limlogpartfunsqrta0}
\end{eqnarray}
We find it useful for the ensuing considerations to utilize the following constraint free representation
\begin{eqnarray}
f_{sq}(\infty)
& = &  \lim_{n,\beta,\beta_z,p\rightarrow\infty} \frac{\mE_G \log \lp \sum_{\bar{\x}\in\mB^n }  \lp
\sum_{ \x\in\mB^n,\z  } \min_{\nu} \lp \sum_{\y\in\mS^m  }e^{\beta \lp f^0 (\z)  -   \y^T G\x  -\nu \bar{\x}^T \x+\nu\bar{\delta}   \rp   }    \rp^{-1}
 \rp^p   \rp  }{np}.
  \label{eq:limlogpartfunsqrt}
\end{eqnarray}
Setting $f_{\bar{\x}} (\x) = f^0(\z) -\nu\bar{\x}^T\x +\nu\bar{\delta}  =  \beta_z(\1^Th(\z-\kappa) - m ) - \y^T\z
+\nu\bar{\x}^T\x -\nu\bar{\delta}  $ and noting statistical equivalence of $G$ and $-G$, we finally have
\begin{eqnarray}
f_{sq}(\infty)
& = &  \lim_{n,\beta,\beta_z,p\rightarrow\infty} \frac{\mE_G  \log \lp \sum_{\bar{\x}\in\mB^n }  \lp
\sum_{ \x\in\mB^n,\z  }  \min_{\nu} \lp \sum_{\y\in\mS^m  }e^{\beta \lp f_{\bar{\x}} (\x)  +   \y^T G\x    \rp   }    \rp^{-1}
 \rp^p   \rp  }{np}.
  \label{eq:limlogpartfunsqrtaa0}
\end{eqnarray}
 We now have all the ingredients to make use of the machinery from \cite{Stojnicnflldp25,Stojnicsflldp25}. The remaining presentation is split into two parts. In the first one we formally show how the above fits within the results of \cite{Stojnicnflldp25,Stojnicsflldp25}. As usual, once such a fit is established, a sizeable set of numerical evaluations is needed to obtain concrete results. In the second part of the remaining presentation, we conduct these evaluations as well.

\section{ABP fit within large deviation sfl RDT}
\label{sec:randlincons}

One first  observes that the ground state free energy from (\ref{eq:limlogpartfunsqrtaa0}) is clearly a function of blirp $\y^TG\x$. To  connect $f_{sq}$ and blirp results from \cite{Stojnicnflldp25,Stojnicsflldp25}, we follow \cite{Stojnichopflrdt23,Stojnicbinperflrdt23,Stojnicnegsphflrdt23} and start by introducing a few needed technical preliminaries. Let $r$ be a positive integer (i.e., let $r\in\mN$) and for $k\in\{1,2,\dots,r+1\}$ consider real scalars $\beta,p\geq 0$, and $s$, sets $\cX,\bar{\cX}\subseteq \mR^n$, and $\cY\subseteq \mR^m$, function $f_S(\cdot):\mR^n\rightarrow R$, and vectors $\p=[\p_0,\p_1,\dots,\p_{r+1}]$, $\q=[\q_0,\q_1,\dots,\q_{r+1}]$, and $\c=[\c_0,\c_1,\dots,\c_{r+1}]$ such that
 \begin{eqnarray}\label{eq:hmsfl2}
1=\p_0\geq \p_1\geq \p_2\geq \dots \geq \p_r\geq \p_{r+1} & = & 0 \nonumber \\
1=\q_0\geq \q_1\geq \q_2\geq \dots \geq \q_r\geq \q_{r+1} & = &  0,
 \end{eqnarray}
with $\c_0=1$, $\c_{r+1}=0$. Also, let ${\mathcal U}_k\triangleq [u^{(4,k)},\u^{(2,k)},\h^{(k)}]$  where the elements of  $u^{(4,k)}\in\mR$, $\u^{(2,k)}\in\mR^m$, and $\h^{(k)}\in\mR^n$ are independent standard normals. Define function
  \begin{eqnarray}\label{eq:fl4}
\psi_{S,\infty}(f_{S},\p,\q,\c,s)  =
 \mE_{G,{\mathcal U}_{r+1}} \frac{1}{n\c_r} \log
\lp \mE_{{\mathcal U}_{r}} \lp \dots \lp \mE_{{\mathcal U}_3}\lp\mE_{{\mathcal U}_2} \lp  \sum_{i_3=1}^{l} \lp  \mE_{{\mathcal U}_1} Z_{S,\infty}  \rp^p \rp^{\c_2}\rp^{\frac{\c_3}{\c_2}}\rp^{\frac{\c_4}{\c_3}} \dots \rp^{\frac{\c_{r}}{\c_{r-1}}}\rp, \nonumber \\
 \end{eqnarray}
with
\begin{eqnarray}\label{eq:fl5}
Z_{S,\infty} & \triangleq & e^{D_{0,S,\infty}} \nonumber \\
 D_{0,S,\infty} & \triangleq  & \max_{\x\in\cX,\|\x\|_2=x} s \max_{\y\in\cY,\|\y\|_2=y}
 \lp \sqrt{n} f_{S}
+\sqrt{n}  y    \lp\sum_{k=1}^{r+1}c_k\h^{(k)}\rp^T\x
+ \sqrt{n} x \y^T\lp\sum_{k=1}^{r+1}b_k\u^{(2,k)}\rp \rp \nonumber  \\
 b_k & \triangleq & b_k(\p,\q)=\sqrt{\p_{k-1}-\p_k} \nonumber \\
c_k & \triangleq & c_k(\p,\q)=\sqrt{\q_{k-1}-\q_k}.
 \end{eqnarray}
With all the above technicalities in place, one can then recall on the following theorem -- certainly one of the key components of sfl LD RDT.
\begin{theorem} \cite{Stojnicsflldp25}
\label{thm:thmsflrdt1}  Consider large $n$ linear regime where  $\alpha=\lim_{n\rightarrow\infty} \frac{m}{n}$ remains constant as  $n$ grows and let $G\in\mR^{m\times n}$ have independent standard normal elements. Let three sets  $\cX,\bar{{\mathcal X}} \subseteq \mR^n$ and $\cY\subseteq \mR^m$ with unit norm elements  and function $f_{\bar{\x}}(\cdot):R^{3n+m}\rightarrow R$ be given. Assume the complete sfl LD RDT frame from \cite{Stojnicsflldp25} and set
\begin{align}\label{eq:thmsflrdt2eq1}
   \psi_{rp} & \triangleq  -\log\lp \sum_{\bar{\x}\in \bar{{\mathcal X}}}
   \lp  \sum_{\x\in {\mathcal X}  }  \lp  \sum_{\y\in {\mathcal Y}  }   e^{ \beta  \lp f_{\bar{\x}}+\y^TG\x \rp } \rp^s \rp^p  \rp
   \qquad  \mbox{(\bl{\textbf{random primal}})} \nonumber \\
   \psi_{rd}(\p,\q,\c,s) & \triangleq    \frac{1}{2}   \sum_{k=1}^{r+1}\Bigg(\Bigg.
   \p_{k-1}\q_{k-1}
   -\p_{k}\q_{k}
  \Bigg.\Bigg)
\c_k \omega(k;p)
  - \psi_{S,\infty}(f_{\bar{\x}},\p,\q,\c,s) \hspace{.23in} \mbox{(\bl{\textbf{fl random dual}})}.
 \end{align}
 For $\hat{\p}_0\rightarrow 1$, $\hat{\q}_0\rightarrow 1$, and $\hat{\c}_0\rightarrow 1$, $\hat{\p}_{r+1}=\hat{\q}_{r+1}=\hat{\c}_{r+1}=0$, let the non-fixed parts of $\hat{\p}$, $\hat{\q}$, and  $\hat{\c}$ satisfy
\begin{eqnarray}\label{eq:thmsflrdt2eq2}
   \frac{d \psi_{rd}(\p,\q,\c,s)}{d\p} =
   \frac{d \psi_{rd}(\p,\q,\c,s)}{d\q} =     \frac{d \psi_{rd}(\p,\q,\c,s)}{d\c} =  0.
 \end{eqnarray}
 Then,
\begin{eqnarray}\label{eq:thmsflrdt2eq3}
    \lim_{n\rightarrow\infty} \frac{\mE_G  \psi_{rp}}{n}
  & = &
 \lim_{n\rightarrow\infty} \psi_{rd}(\hat{\p},\hat{\q},\hat{\c},s) \qquad \mbox{(\bl{\textbf{strong sfl random duality}})},\nonumber \\
 \end{eqnarray}
where $\psi_{S,\infty}(\cdot)$ is as in (\ref{eq:fl4})-(\ref{eq:fl5}).
 \end{theorem}
\begin{proof}
  Follows immediately from Corollary 1 in \cite{Stojnicsflldp25}.
 \end{proof}

\subsection{Practical realization}
\label{sec:prac}

While the  mathematical form from Theorem \ref{thm:thmsflrdt1} is fairly elegant it becomes  practically useful only if all associated quantities can be evaluated. As we will see below, already on the second level of lifting ($r=2$), one obtains results that are not only in an excellent agreement with replica predictions but remarkably accurately agree with the presumed BP's algorithmic hardness.

After taking binary cube sets $\cX=\bar{{\mathcal X}}=\mB^n = \{-\frac{1}{\sqrt{n}},\frac{1}{\sqrt{n}}\}^n$ and the unit sphere $\cY=\mS^m$ together with $f_{\bar{\x}} (\x) =  \beta_z(\1^Th(\z-\kappa) - m ) -\y^T\z
+\nu\bar{\x}^T\x -\nu\bar{\delta}  $, $\beta,\beta_z, p\rightarrow\infty$, and $\c_k\rightarrow\frac{\c_k}{p}, k>1$, in Theorem \ref{thm:thmsflrdt1}, we observe that the \emph{random dual} becomes
\begin{align}\label{eq:prac1}
    \psi_{rd}(\p,\q,\c,s) & \triangleq    \frac{1}{2}    \sum_{k=1}^{r+1}\Bigg(\Bigg.
   \p_{k-1}\q_{k-1}
   -\p_{k}\q_{k}
  \Bigg.\Bigg)
\c_k
  - \psi_{S,\infty}(f_{\bar{\x}},\p,\q,\c,s). \nonumber \\
  & =   \frac{1}{2}    \sum_{k=1}^{r+1}\Bigg(\Bigg.
   \p_{k-1}\q_{k-1}
   -\p_{k}\q_{k}
  \Bigg.\Bigg)
\c_k
  - \frac{1}{n}\varphi(D^{(bin)}(\bar{\x},s)) - \frac{1}{n}\varphi(D^{(sph)}(\bar{\x},s)), \nonumber \\
  \end{align}
where analogously to (\ref{eq:fl4})-(\ref{eq:fl5})
  \begin{eqnarray}\label{eq:prac2}
\varphi(D,\c) & = &
 \mE_{G,{\mathcal U}_{r+1}} \frac{1}{\c_r} \log
\lp \mE_{{\mathcal U}_{r}} \lp \dots \lp \mE_{{\mathcal U}_3}\lp\mE_{{\mathcal U}_2} \lp
  \max_{\bar{\x}\in\bar{{\mathcal X}}} \mE_{{\mathcal U}_{1}} e^{D}  \rp^{\c_2}\rp^{\frac{\c_3}{\c_2}}\rp^{\frac{\c_4}{\c_3}} \dots \rp^{\frac{\c_{r}}{\c_{r-1}}}\rp, \nonumber \\
  \end{eqnarray}
and
\begin{eqnarray}\label{eq:prac3}
D^{(bin)}(\bar{\x},s) & = & \max_{\x\in\cX} \lp   s\sqrt{n}
   \lp \lp\sum_{k=1}^{r+1}c_k\h^{(k)}\rp^T\x
+\nu \bar{\x}^T\x -\nu \bar{\delta}
\rp
     \rp
    \nonumber \\
  D^{(sph)}(s) & =  &  \max_{\z} s \max_{\y\in\cY}
\lp \sqrt{n} \lp
\beta_{\z} \sum_{i_1=1}^{l} (1-h(\z_{i_1}-\kappa))  +  \y^T\lp\sum_{k=1}^{r+1}b_k\u^{(2,k)}  -\z  \rp \rp \rp.
 \end{eqnarray}
One then also has
\begin{eqnarray}\label{eq:prac4}
D^{(bin)}(\bar{\x},s) & = & \max_{\x\in\cX}   \lp s\sqrt{n}    \lp  \lp\sum_{k=1}^{r+1}c_k\h^{(k)}\rp^T  +\nu\bar{\x} \rp \x \rp =
       \sum_{i=1}^n \left |s \lp  \lp\sum_{k=1}^{r+1}c_k\h_i^{(k)}\rp  +\nu\bar{\x}_i  \rp   \right |
       \nonumber \\
& = &     \sum_{i=1}^n \left | \lp\sum_{k=1}^{r+1}c_k\h_i^{(k)}\rp   +\nu\bar{\x}_i   \right |
=  \sum_{i=1}^n D^{(bin)}_i(\bar{\x},c_k) - \nu\bar{\delta} , \nonumber \\
 \end{eqnarray}
with
\begin{eqnarray}\label{eq:prac5}
D^{(bin)}_i(\bar{\x},c_k)=\left |  \lp\sum_{k=1}^{r+1}c_k\h_i^{(k)}\rp    +\nu\bar{\x}_i  \right |.
\end{eqnarray}
Moreover,
  \begin{eqnarray}\label{eq:prac6}
\varphi(D^{(bin)}(\bar{\x},s),\c) & = &
n \mE_{G,{\mathcal U}_{r+1}} \frac{1}{\c_r} \log
\lp \mE_{{\mathcal U}_{r}} \lp \dots \lp \mE_{{\mathcal U}_3}\lp\mE_{{\mathcal U}_2} \lp
  \max_{\bar{\x}_i \in\bar{{\mathcal X}}} \mE_{{\mathcal U}_{1}} e^{D_1^{(bin)}(\bar{\x},c_k)}  \rp^{\c_2}\rp^{\frac{\c_3}{\c_2}}\rp^{\frac{\c_4}{\c_3}} \dots \rp^{\frac{\c_{r}}{\c_{r-1}}}\rp
\nonumber \\
& &
 - \nu\bar{\delta}
\nonumber \\
    & =  &  n\varphi(D_1^{(bin)}(\bar{\x},c_k),\c)  - \nu\bar{\delta}. \nonumber \\
   \end{eqnarray}
We then analogously find
\begin{eqnarray}\label{eq:prac7aa0}
   D^{(sph)}(s) &  =  &    s  \sqrt{n}   \left \| \min \lp - \kappa\1 +\sum_{k=2}^{r+1}b_k\u^{(2,k)},0 \rp  \right \|_2.
 \end{eqnarray}
 Given statistical symmetry, for all practical purposes we can alternatively take
\begin{eqnarray}\label{eq:prac7aa0}
   D^{(sph)}(s) &  =  &    s  \sqrt{n}   \left \| \max \lp \kappa\1 +\sum_{k=2}^{r+1}b_k\u^{(2,k)},0 \rp  \right \|_2.
 \end{eqnarray}
Utilization of the \emph{square root trick} introduced  in \cite{StojnicMoreSophHopBnds10,StojnicLiftStrSec13,StojnicGardSphErr13,StojnicGardSphNeg13} gives as in \cite{Stojnicbinperflrdt23}'s (30) nd (31)
 \begin{eqnarray}\label{eq:prac9}
   D^{(sph)}(s)
   & =  &  s \min_{\gamma_{sq}} \lp \sum_{i=1}^{m} D_i^{(sph)}(b_k)+\gamma_{sq}n \rp, \nonumber \\
 \end{eqnarray}
with
\begin{eqnarray}\label{eq:prac10}
   D_i^{(sph)}(b_k)= \frac{\max \lp \kappa + \sum_{k=2}^{r+1}b_k\u_i^{(2,k)},0  \rp^2}{4\gamma_{sq}}.
 \end{eqnarray}

After  taking $s=-1$ one can finally connect the ground state energy, $f_{sq}$ given in (\ref{eq:limlogpartfunsqrt}), and the random primal, $\psi_{rp}(\cdot)$, given in Theorem \ref{thm:thmsflrdt1}. In particular, for optimal $\hat{\nu}$ and $\hat{\gamma}_{sq}$ (those that maximize and minimize $\psi_{rd}$, respectively) we have
\begin{eqnarray}\label{eq:negprac11}
f_{sq}(\infty)
& = &
\lim_{n,\beta,\beta_z,p\rightarrow\infty} \frac{\mE_G  \log \lp \sum_{\bar{\x}\in\mB^n }  \lp
\sum_{ \x\in\mB^n,\z  }  \min_{\nu} \lp \sum_{\y\in\mS^m  }e^{\beta \lp f_{\bar{\x}} (\x)  +   \y^T G\x    \rp   }    \rp^{-1}
 \rp^p   \rp  }{np}
\nonumber \\
  & = &
 -   \lim_{n,\beta,\beta_z,p\rightarrow\infty} \frac{\mE_G \max_{\nu} \psi_{rp}}{n}
   =
 - \lim_{n,\beta,\beta_z,p\rightarrow\infty} \max_{\nu} \min_{\gamma_{sq}}\psi_{rd}(\hat{\p},\hat{\q},\hat{\c},-1),
\end{eqnarray}
where the non-fixed parts of $\hat{\p}$, $\hat{\q}$, and  $\hat{\c}$ satisfy
\begin{eqnarray}\label{eq:negprac12}
   \frac{d \psi_{rd}(\p,\q,\c,-1)}{d\p} =
   \frac{d \psi_{rd}(\p,\q,\c,-1)}{d\q} =
   \frac{d \psi_{rd}(\p,\q,\c,-1)}{d\c} =  0.
 \end{eqnarray}
Relying on (\ref{eq:prac1})-(\ref{eq:prac10}), one then finds
 \begin{eqnarray}
 \lim_{n\rightarrow\infty} \psi_{rd}(\hat{\p},\hat{\q},\hat{\c},-1) =  \bar{\psi}_{rd}(\hat{\p},\hat{\q},\hat{\c},\hat{\nu},\hat{\gamma}_{sq},-1),
  \label{eq:negprac12a}
\end{eqnarray}
where
\begin{eqnarray}\label{eq:negprac13}
    \bar{\psi}_{rd}(\p,\q,\c,\nu,\gamma_{sq},-1)   & = &  \frac{1}{2}    \sum_{k=1}^{r+1}\Bigg(\Bigg.
   \p_{k-1}\q_{k-1}
   -\p_{k}\q_{k}
  \Bigg.\Bigg)
\c_k
\nonumber \\
& & +\nu\bar{\delta}  - \varphi(D_1^{(bin)}(\bar{\x},c_k(\p,\q))) +\gamma_{sq}- \alpha\varphi(-D_1^{(sph)}(b_k(\p,\q))).\nonumber \\
  \end{eqnarray}
Combining  (\ref{eq:negprac11}), (\ref{eq:negprac12a}), and (\ref{eq:negprac13}) we obtain
 \begin{eqnarray}
f_{sq}(\infty)
& = &
\lim_{n,\beta,\beta_z,p\rightarrow\infty} \frac{\mE_G  \log \lp \sum_{\bar{\x}\in\mB^n }  \lp
\sum_{ \x\in\mB^n,\z  }  \min_{\nu} \lp \sum_{\y\in\mS^m  }e^{\beta \lp f_{\bar{\x}} (\x)  +   \y^T G\x    \rp   }    \rp^{-1}
 \rp^p   \rp  }{np}
\nonumber \\
    &  = &
 -\lim_{n\rightarrow\infty} \psi_{rd}(\hat{\p},\hat{\q},\hat{\c},-1)
 =  - \bar{\psi}_{rd}(\hat{\p},\hat{\q},\hat{\c},\hat{\nu},\hat{\gamma}_{sq},-1) \nonumber \\
 & = & -\Bigg ( \Bigg .  \frac{1}{2}    \sum_{k=1}^{r+1}\Bigg(\Bigg.
   \hat{\p}_{k-1}\hat{\q}_{k-1}
   -\hat{\p}_{k}\hat{\q}_{k}
  \Bigg.\Bigg)
\hat{\c}_k
\nonumber \\
& &
+\hat{\nu}\bar{\delta}
  - \varphi(D_1^{(bin)}(\bar{\x}\,c_k(\hat{\p},\hat{\q}))) + \hat{\gamma}_{sq} - \alpha\varphi(-D_1^{(sph)}(\b_k(\hat{\p},\hat{\q})))\Bigg .\Bigg ).
  \label{eq:negprac18}
\end{eqnarray}

We summarize the above discussion in the following theorem.
\begin{theorem}
  \label{thme:negthmprac1}
Consider large $n$ linear regime with $\alpha=\lim_{n\rightarrow\infty} \frac{m}{n}$ and  assume complete sfl LD RDT setup of \cite{Stojnicsflldp25}. Let
  $\varphi(\cdot)$ and $\bar{\psi}(\cdot)$ be as in (\ref{eq:prac2}) and (\ref{eq:negprac13}), respectively. Also, take the ``fixed'' parts of $\hat{\p}$, $\hat{\q}$, and $\hat{\c}$ as $\hat{\p}_1\rightarrow 1$, $\hat{\q}_1\rightarrow 1$, $\hat{\c}_1\rightarrow 1$, $\hat{\p}_{r+1}=\hat{\q}_{r+1}=\hat{\c}_{r+1}=0$, and let their ``non-fixed'' parts  satisfy
  \begin{eqnarray}\label{eq:negthmprac1eq1}
   \frac{d \bar{\psi}_{rd}(\p,\q,\c,\nu,\gamma_{sq},-1)}{d\p} & = &
   \frac{d \bar{\psi}_{rd}(\p,\q,\c,\nu,\gamma_{sq},-1)}{d\q} =     \frac{d \bar{\psi}_{rd}(\p,\q,\c,nu,\gamma_{sq},-1)}{d\c} =  0 \nonumber \\
   \frac{d \bar{\psi}_{rd}(\p,\q,\c,\nu,\gamma_{sq},-1)}{d\nu} & = &
   \frac{d \bar{\psi}_{rd}(\p,\q,\c,\nu,\gamma_{sq},-1)}{d\gamma_{sq}} =  0.
 \end{eqnarray}
Moreover, set
\begin{eqnarray}\label{eq:prac17}
c_k(\hat{\p},\hat{\q})  & = & \sqrt{\hat{\q}_{k-1}-\hat{\q}_k} \nonumber \\
b_k(\hat{\p},\hat{\q})  & = & \sqrt{\hat{\p}_{k-1}-\hat{\p}_k}.
 \end{eqnarray}
 Then
 \begin{eqnarray}
f_{sq}(\infty)
& = &   -
  \Bigg.\Bigg (
 \frac{1}{2}    \sum_{k=1}^{r+1}\Bigg(\Bigg.
   \hat{\p}_{k-1}\hat{\q}_{k-1}
   -\hat{\p}_{k}\hat{\q}_{k}
  \Bigg.\Bigg)
\hat{\c}_k
\nonumber \\
& &
+\hat{\nu}\bar{\delta}
  - \varphi(D_1^{(bin)}(\bar{\x},c_k(\hat{\p},\hat{\q})),\hat{\c}) + \hat{\gamma}_{sq} - \alpha\varphi(-D_1^{(sph)}(\bar{\x},b_k(\hat{\p},\hat{\q})),\hat{\c})
    \Bigg.\Bigg).
  \label{eq:negthmprac1eq2}
\end{eqnarray}
\end{theorem}
\begin{proof}
Follows immediately from the previous discussion and Theorem \ref{thm:thmsflrdt1}.
\end{proof}

\subsection{Numerical evaluations}
\label{sec:nuemrical}

To have the results of Theorem \ref{thme:negthmprac1} in full power one needs to conduct underlying numerical evaluations. A particular value of $\kappa$  is also needed to obtain concrete practically useful values. We take $\kappa=0$ which corresponds to the most famous, so-called zero-threshold, scenario. Also, as mentioned earlier, we immediately present the results for the second level of lifting since that is the lowest level where the key LE features -- particularly relevant for BP algorithmic hardness and presumable existence of the associated computational gap -- appear.

\subsubsection{$r=2$ -- second level of lifting}
\label{sec:secondlev}

We have, $r=2$,  $\hat{\p}_0 = \hat{\q}_0=1$, $\hat{\p}_{r+1}=\hat{\p}_{3}=\hat{\q}_{r+1}=\hat{\q}_{3}=0$, and $\c_1\rightarrow 1$ but in general  $\p_1\neq0$, $\p_2\neq0$,  $\q_1\neq0$, $\q_2\neq0$, and $\hat{\c}_{2}\neq 0$. Then we can write
\begin{eqnarray}\label{eq:negprac24}
    \bar{\psi}_{rd}(\p,\q,\c,\nu,\gamma_{sq},-1)
    \hspace{-.05in}
      & = &  \frac{1}{2}
\beta^2(1-\p_1\q_1)\c_1
+ \frac{1}{2}\beta^2(\p_1\q_1-\p_2\q_2)\c_2
\nonumber \\
& &
+\nu  - \frac{1}{\c_2}\mE_{{\mathcal U}_3}\log\lp \mE_{{\mathcal U}_2}
  \lp
  \max_{\bar{\x}_1=\pm 1 }
\lp \mE_{{\mathcal U}_1} e^{\beta |\sqrt{1-\q_1}\h_1^{(1)} +\sqrt{\q_1-\q_2}\h_1^{(2)}+\sqrt{\q_2}\h_1^{(3)}  +\nu\bar{\x}_1   |} \rp  \rp^{\c_2}
\rp \nonumber \\
& &   + \gamma_{sq}
 -\alpha\frac{1}{\c_2}\mE_{{\mathcal U}_3} \log\lp \mE_{{\mathcal U}_2}\lp \mE_{{\mathcal U}_1} e^{-\beta\frac{\max(\kappa+ \sqrt{1-\p_1}\u_1^{(2,1)}+\sqrt{\p_1-\p_2}\u_1^{(2,2) } +\sqrt{\p_2}\u_1^{(2,3)} ,0)^2} {4\gamma_{sq}}}\rp^{\c_2}\rp.
 \nonumber \\
    \end{eqnarray}
One then first finds
\begin{eqnarray}\label{eq:negprac24a0}
f_{(z)}^{(2)} & = & \mE_{{\mathcal U}_1} e^{\beta|\sqrt{1-\q_1}\h_1^{(1)} +\sqrt{\q_1-\q_2}\h_1^{(2)}+\sqrt{\q_2}\h_1^{(3)} +\nu\bar{\x}_1   |}
  \nonumber \\
 & = &  \frac{1}{2}
 e^{\frac{(1-\q_1)\beta^2}{2}}
 \nonumber  \\
 & & \times
 \Bigg(\Bigg.
 e^{-\beta \lp  \sqrt{\q_1-\q_2}\h_1^{(2)} + \sqrt{\q_2}\h_1^{(3)}  +\nu\bar{\x}_1 \rp }
 \erfc\lp - \lp\beta\sqrt{1-\q_1}-\frac{  \sqrt{\q_1-\q_2}\h_1^{(2)} +\sqrt{\q_2}\h_1^{(3)}  +\nu\bar{\x}_1   }  {\sqrt{1-\q_1}}\rp\frac{1}{\sqrt{2}}\rp \nonumber \\
& &  + e^{\beta \lp  \sqrt{\q_1-\q_2}\h_1^{(2)} + \sqrt{\q_2}\h_1^{(3)}  +\nu\bar{\x}_1 \rp }
   \erfc  \Bigg.\Bigg(  -  \Bigg.\Bigg( \beta\sqrt{1-\q_1}+\frac{  \sqrt{\q_1-\q_2}\h_1^{(2)} + \sqrt{\q_2}\h_1^{(3)}  +\nu\bar{\x}_1   }   {\sqrt{1-\q_1}} \Bigg.\Bigg) \frac{1}{\sqrt{2}}
    \Bigg.\Bigg)
   \Bigg.\Bigg),\nonumber \\
     \end{eqnarray}
 and
\begin{equation}\label{eq:negprac24a1}
\mE_{{\mathcal U}_3}\log\lp \mE_{{\mathcal U}_2}
  \lp
  \max_{\bar{\x}_1=\pm 1 }
\lp \mE_{{\mathcal U}_1} e^{\beta |\sqrt{1-\q_1}\h_1^{(1)} +\sqrt{\q_1-\q_2}\h_1^{(2)}+\sqrt{\q_2}\h_1^{(3)}  +\nu\bar{\x}_1   |} \rp \rp^{\c_2}
\rp =  \mE_{{\mathcal U}_3}\log   \lp  \mE_{{\mathcal U}_2} \lp   \max_{\bar{\x}_1=\pm 1 }
 f_{(z)}^{(2)} \rp^{\c_2}  \rp.
    \end{equation}
Also, after setting
\begin{eqnarray}\label{eq:negprac24a2}
\bar{h} & = &  -\frac{\sqrt{\p_1-\p_2}\u_1^{(2,2)}+\sqrt{\p_2}\u_1^{(2,3)}+\kappa}{\sqrt{1-\p_1}}    \nonumber \\
\bar{B} & = & \frac{\c_2}{4\gamma_{sq}} 
\nonumber \\
\bar{C} & = & \sqrt{\p_1-\p_2}\u_1^{(2,2)}+\sqrt{\p_2}\u_1^{(2,3)} + \kappa \nonumber \\
f_{(zd)}^{(2,f)}& = & \frac{e^{-\frac{\bar{B}\bar{C}^2}{2(1-\p_1)\bar{B} + 1}}}{2\sqrt{2(1-\p_1)\bar{B} + 1}}
\erfc\lp\frac{\bar{h}}{\sqrt{4(1-\p_1)\bar{B} + 2}}\rp
\nonumber \\
f_{(zu)}^{(2,f)}& = & \frac{1}{2}\erfc\lp-\frac{\bar{h}}{\sqrt{2}}\rp,  
   \end{eqnarray}
we find
\begin{eqnarray}\label{eq:negprac24a3}
\mE_{{\mathcal U}_3} \log\lp \mE_{{\mathcal U}_2}\lp \mE_{{\mathcal U}_1} e^{-\beta\frac{\max(\sqrt{1-\p_1}\u_1^{(2,1)}+\sqrt{\p_1-\p_2}\u_1^{(2,2) } +\sqrt{\p_2}\u_1^{(2,3)} ,0)^2} {4\gamma_{sq}}}\rp^{\c_2}\rp
=   \mE_{{\mathcal U}_3} \log   \lp  \mE_{{\mathcal U}_2} \lp  f_{(zd)}^{(2,f)}+f_{(zu)}^{(2,f)}\rp^{\c_2} \rp.
    \end{eqnarray}
Differentiation/optimization give scaling
$\q_1\beta^2\rightarrow \q_1^{(s)}$, $\q_2\beta^2\rightarrow \q_2^{(s)}$, and $\gamma_{sq}\rightarrow 0$, and
\begin{eqnarray}\label{eq:negprac24a4}
f_{(z)}^{(2)}
 & \rightarrow &
 e^{\frac{\beta^2-\q_1^{(s)}}{2}}
 \Bigg(\Bigg.
 e^{-  \lp  \sqrt{\q_1^{(s)}-\q_2^{(s)}}\h_1^{(3)}  + \sqrt{\q_2^{(s)}}\h_1^{(3)}   + \nu\bar{\x}_1  \rp }
    + e^{   \lp  \sqrt{\q_1^{(s)}-\q_2^{(s)}}\h_1^{(3)}  + \sqrt{\q_2^{(s)}}\h_1^{(3)}   + \nu\bar{\x}_1  \rp  }
    \Bigg.\Bigg)
    \nonumber \\
 & \rightarrow &
 e^{\frac{\beta^2-\q_1^{(s)}}{2}}
2 \cosh \lp  \sqrt{\q_1^{(s)}-\q_2^{(s)}}\h_1^{(3)}  + \sqrt{\q_2^{(s)}}\h_1^{(3)}   + \nu\bar{\x}_1  \rp  .
     \end{eqnarray}
After observing
\begin{eqnarray}\label{eq:negprac24a5}
\bar{h} & = &  -\frac{\sqrt{\p_1-\p_2}\u_1^{(2,2)}+\sqrt{\p_2}\u_1^{(2,3)}+\kappa}{\sqrt{1-\p_1}}    \nonumber \\
\bar{B} & = & \frac{\c_2}{4\gamma_{sq}} \rightarrow \infty
\nonumber \\
\bar{C} & = & \sqrt{\p_1-\p_2}\u_1^{(2,2)}+\sqrt{\p_2}\u_1^{(2,3)} + \kappa \nonumber \\
f_{(zd)}^{(2,f)}& = & \frac{e^{-\frac{\bar{B}\bar{C}^2}{2(1-\p_1)\bar{B} + 1}}}{2\sqrt{2(1-\p_1)\bar{B} + 1}}
\erfc\lp\frac{\bar{h}}{\sqrt{4(1-\p_1)\bar{B} + 2}}\rp
\rightarrow 0
\nonumber \\
 f_{(zu)}^{(2,f)}& = & \frac{1}{2}\erfc\lp-\frac{\bar{h}}{\sqrt{2}}\rp
\rightarrow \frac{1}{2}\erfc\lp \frac{  \sqrt{\p_1-\p_2}\u_1^{(2,2)}+\sqrt{\p_2}\u_1^{(2,3)}+\kappa  }{\sqrt{2}\sqrt{1-\p_1}} \rp.
   \end{eqnarray}
and combining (\ref{eq:negprac24}), (\ref{eq:negprac24a1}), (\ref{eq:negprac24a3}), (\ref{eq:negprac24a4}), and (\ref{eq:negprac24a5}), we find
\begin{eqnarray}\label{eq:negprac24a6}
-f_{sq}^{(2,f)}(\infty) & = &     \bar{\psi}_{rd}(\p,\q,\c,\gamma_{sq},-1) \nonumber \\
  & = &  \frac{1}{2}
(\beta^2-\p_1\q_1^{(s)})
+ \frac{1}{2}(\p_1\q_1^{(s)}-\p_2\q_2^{(s)})\c_2
- \frac{\beta^2-\q_1^{(s)}}{2}
\nonumber \\
& & +\nu  - \frac{1}{\c_2}
 \mE_{{\mathcal U}_3}\log   \lp  \mE_{{\mathcal U}_3} \lp   \max_{\bar{\x}_1=\pm 1 }
 f_{(z)}^{(2)} \rp^{\c_2}  \rp
  + \gamma_{sq}
 -\alpha\frac{1}{\c_2}
  \mE_{{\mathcal U}_3} \log   \lp  \mE_{{\mathcal U}_3} \lp  f_{(zd)}^{(2,f)}+f_{(zu)}^{(2,f)}\rp^{\c_2} \rp
   \nonumber\\
  & \rightarrow &
\frac{(1-\p_1)\q_1^{(s)}}{2}
+ \frac{1}{2}(\p_1\q_1^{(s)}-\p_2\q_2^{(s)})\c_2
\nonumber \\
& &
+\nu
-\frac{1}{\c_2}
 \mE_{{\mathcal U}_3}\log   \lp  \mE_{{\mathcal U}_2} \lp   \max_{\bar{\x}_1=\pm 1 }
 2 \cosh \lp  \sqrt{\q_1^{(s)}-\q_2^{(s)}}\h_1^{(2)}  + \sqrt{\q_2^{(s)}}\h_1^{(3)}   + \nu\bar{\x}_1  \rp
  \rp^{\c_2}  \rp
   \nonumber \\
& &
  + \gamma_{sq}
 -\alpha\frac{1}{\c_2}
  \mE_{{\mathcal U}_3} \log   \lp  \mE_{{\mathcal U}_2} \lp  \frac{1}{2}\erfc\lp \frac{  \sqrt{\p_1-\p_2}\u_1^{(2,2)}+\sqrt{\p_2}\u_1^{(2,3)}+\kappa  }{\sqrt{2}\sqrt{1-\p_1}} \rp     \rp^{\c_2} \rp
     \nonumber\\
  & \rightarrow &
\frac{(1-\p_1)\q_1^{(s)}}{2}
+ \frac{1}{2}(\p_1\q_1^{(s)}-\p_2\q_2^{(s)})\c_2
\nonumber \\
& &
+\nu
-\frac{1}{\c_2}
 \mE_{{\mathcal U}_3}\log   \lp  \mE_{{\mathcal U}_2} \lp
 2 \cosh \lp \left | \sqrt{\q_1^{(s)}-\q_2^{(s)}}\h_1^{(2)}  + \sqrt{\q_2^{(s)}}\h_1^{(3)} \right |  + |\nu|  \rp
  \rp^{\c_2}  \rp
   \nonumber \\
& &
  + \gamma_{sq}
 -\alpha\frac{1}{\c_2}
  \mE_{{\mathcal U}_3} \log   \lp  \mE_{{\mathcal U}_2} \lp  \frac{1}{2}\erfc\lp \frac{  \sqrt{\p_1-\p_2}\u_1^{(2,2)}+\sqrt{\p_2}\u_1^{(2,3)}+\kappa  }{\sqrt{2}\sqrt{1-\p_1}} \rp     \rp^{\c_2} \rp. \nonumber\\
    \end{eqnarray}

\noindent \underline{\textbf{\emph{$\p$ - derivatives}}}

\noindent For $\p_1$ we first note
\begin{eqnarray}\label{eq:pder1}
\frac{d\bar{h} } {d\p_1}   & = &  -  \frac{ 1/2/\sqrt{\p_1-\p_2} \u_1^{(2,2)}   } {\sqrt{1-\p_1} }  - 1/2\frac{\sqrt{\p_1-\p_2} \u_1^{(2,2 )}  + \sqrt{\p_2} \u_1^{(2,3 )} + \kappa  } { \sqrt{1-\p_1}^3 }   \nonumber \\
 \frac{d f_{(zu)}^{(2,f)} } {d\p_1} & = &
   - 1/2 \lp -1/\sqrt{2} \frac{d\bar{h} } {d\p_1}  \rp  \lp\frac{2}{\sqrt{\pi}} e^{-(\bar{h}/\sqrt{2}  )^2} \rp.
   \end{eqnarray}
We then obtain
\begin{eqnarray}\label{eq:pder2}
  \frac{d  \bar{\psi}_{rd}(\p,\q,\c,\nu,\gamma_{sq},-1)}{d\p_1}
 &\rightarrow &
\frac{(\c_2-1)}{2}\q_1^{(s)}
 -\alpha\frac{1}{\c_2}
  \mE_{{\mathcal U}_3} \frac{ \mE_{{\mathcal U}_2} \lp \c_2 \lp f_{(zu)}^{(2,f)} \rp^{\c_2-1}   \frac{d f_{(zu)}^{(2,f)} } {d\p_1}   \rp   }
    { \mE_{{\mathcal U}_2} \lp  f_{(zu)}^{(2,f)}     \rp^{\c_2}  }.
    \end{eqnarray}
For $\p_2$ we further note
\begin{eqnarray}\label{eq:pder3}
\frac{d\bar{h} } {d\p_2}   & = &    \frac{ - 1/2/\sqrt{\p_1-\p_2} \u_1^{(2,2 )}  + 1/2/\sqrt{\p_2} \u_1^{(2,3 )}  } { \sqrt{1-\p_1} }   \nonumber \\
 \frac{d f_{(zu)}^{(2,f)} } {d\p_2} & = &
   - 1/2 \lp -1/\sqrt{2} \frac{d\bar{h} } {d\p_2}  \rp  \lp\frac{2}{\sqrt{\pi}} e^{-(\bar{h}/\sqrt{2}  )^2} \rp,
   \end{eqnarray}
and obtain
\begin{eqnarray}\label{eq:pder4}
  \frac{d  \bar{\psi}_{rd}(\p,\q,\c,\nu,\gamma_{sq},-1)}{d\p_2}
 &\rightarrow &
-\frac{\c_2}{2}\q_1^{(s)}
 -\alpha\frac{1}{\c_2}
  \mE_{{\mathcal U}_3} \frac{ \mE_{{\mathcal U}_2}  \lp \c_2 \lp f_{(zu)}^{(2,f)} \rp^{\c_2-1}   \frac{d f_{(zu)}^{(2,f)} } {d\p_2}   \rp   }
    {  \mE_{{\mathcal U}_2} \lp  f_{(zu)}^{(2,f)}     \rp^{\c_2}  }.
    \end{eqnarray}

\noindent \underline{\textbf{\emph{$\q$ - derivatives}}}

\noindent After setting
\begin{eqnarray}\label{eq:pder5}
f_{(zc)}^{(2)} =  2 \cosh \lp \left | \sqrt{\q_1^{(s)}-\q_2^{(s)}}\h_1^{(2)}  + \sqrt{\q_2^{(s)}}\h_1^{(3)} \right |  + |\nu|  \rp,
    \end{eqnarray}
 we have (modulo continuity)
\begin{eqnarray}\label{eq:pder6}
 \frac{d f_{(zc)}^{(2)} } {d\q_1^{(s)} } & = &
 \mbox{sign}    \lp  \sqrt{\q_1^{(s)} -\q_2^{(s)} }\h_1^{(2)}  + \sqrt{\q_2^{(s)} } \h_1^{(2)}  \rp
 \lp 1/2/\sqrt{\q_1^{(s)} -\q_2^{(s)} }\h_1^{(2)}    \rp
 \nonumber \\
 & & \times
 2 \sinh \lp \left | \sqrt{\q_1^{(s)} -\q_2^{(s)} }\h_1^{(2)} + \sqrt{\q_2^{(s)} } \h_1^{(3)} \right |+ |\nu|  \rp,
\end{eqnarray}
and
\begin{eqnarray}\label{eq:pder7}
  \frac{d  \bar{\psi}_{rd}(\p,\q,\c,\nu,\gamma_{sq},-1)}{d\q_1^{(s)}  }
 &\rightarrow &
\frac{(1-\p_1)}{2}
+\frac{\p_1}{2}\c_2
 - \frac{1}{\c_2}
  \mE_{{\mathcal U}_3} \frac{ \mE_{{\mathcal U}_2} \lp  \c_2 \lp f_{(zc)}^{(2)} \rp^{\c_2-1}   \frac{d f_{(zc)}^{(2)} } {d\q_1^{(s)}}    \rp  }
    {  \mE_{{\mathcal U}_2} \lp  f_{(zc)}^{(2)}     \rp^{\c_2}  }.
    \end{eqnarray}
Analogously, we then have for $\q_2^{(s)}$
\begin{eqnarray}\label{eq:pder8}
 \frac{d f_{(zc)}^{(2)} } {d\q_2^{(s)} } & = &
 \mbox{sign}    \lp  \sqrt{\q_1^{(s)} -\q_2^{(s)} }\h_1^{(2)}  + \sqrt{\q_2^{(s)} } \h_1^{(2)}  \rp
 \lp -1/2/\sqrt{\q_1^{(s)} -\q_2^{(s)} }\h_1^{(2)} + 1/2/\sqrt{\q_2^{(s)} } \h_1^{(3)}  \rp
 \nonumber \\
 & & \times
 2 \sinh \lp \left | \sqrt{\q_1^{(s)} -\q_2^{(s)} }\h_1^{(2)} + \sqrt{\q_2^{(s)} } \h_1^{(3)} \right |+ |\nu|  \rp,
\end{eqnarray}
and
\begin{eqnarray}\label{eq:pder9}
  \frac{d  \bar{\psi}_{rd}(\p,\q,\c,\nu,\gamma_{sq},-1)} {d\q_1^{(s)}  }
 &\rightarrow &
 -\frac{\p_2}{2}\c_2
 - \frac{1}{\c_2}
  \mE_{ {\mathcal U}_3 } \frac{ \mE_{{\mathcal U}_2} \lp  \c_2 \lp f_{(zc)}^{(2)} \rp^{\c_2-1}   \frac{d f_{(zc)}^{(2)} } {d\q_2^{(s)}  }   \rp  }
    {  \mE_{{\mathcal U}_2} \lp  f_{(zc)}^{(2)}     \rp^{\c_2}  }.
    \end{eqnarray}

\noindent \underline{\textbf{\emph{$\c_2$ and $\nu$ - derivatives}}}

\noindent For $\c_2$ we have

\begin{eqnarray}\label{eq:pder10}
  \frac{d  \bar{\psi}_{rd}(\p,\q,\c,\nu,\gamma_{sq},-1)}{d\c_2 }
  & \rightarrow &
  \frac{d}{d\c_2}
   \Bigg .\Bigg (
\frac{(1-\p_1)\q_1^{(s)}}{2}
+ \frac{1}{2}(\p_1\q_1^{(s)}-\p_2\q_2^{(s)})\c_2
\nonumber \\
& &
+\nu
-\frac{1}{\c_2}
 \mE_{{\mathcal U}_3}\log   \lp  \mE_{{\mathcal U}_2} \lp
f_{(zc)}^{(2)}
  \rp^{\c_2}  \rp
   + \gamma_{sq}
 -\alpha\frac{1}{\c_2}
  \mE_{{\mathcal U}_3} \log   \lp  \mE_{{\mathcal U}_2} \lp
  f_{(zu)}^{(2,f)}
      \rp^{\c_2} \rp \Bigg .\Bigg ). \nonumber\\
 & \rightarrow &
   \frac{1}{2}(\p_1\q_1^{(s)}-\p_2\q_2^{(s)})
\nonumber \\
& &
 +\frac{1}{\c_2^2}
 \mE_{{\mathcal U}_3}\log   \lp  \mE_{{\mathcal U}_2} \lp
f_{(zc)}^{(2)}
  \rp^{\c_2}  \rp
- \frac{1}{\c_2}
  \mE_{{\mathcal U}_3} \frac{ \mE_{{\mathcal U}_2} \lp \log  \lp f_{(zc)}^{(2)} \rp  \lp f_{(zc)}^{(2)} \rp^{\c_2}  \rp  }
    {  \mE_{{\mathcal U}_2} \lp  f_{(zc)}^{(2)}     \rp^{\c_2}  }
 \nonumber \\
 & &
 +\frac{\alpha}{\c_2^2}
 \mE_{{\mathcal U}_3}\log   \lp  \mE_{{\mathcal U}_2} \lp
f_{(zu)}^{(2,f)}
  \rp^{\c_2}  \rp
- \frac{\alpha}{\c_2}
  \mE_{{\mathcal U}_3} \frac{ \mE_{{\mathcal U}_2} \lp \log  \lp f_{(zu)}^{(2,f)} \rp  \lp f_{(zu)}^{(2,f)} \rp^{\c_2}  \rp  }
    {  \mE_{{\mathcal U}_2} \lp  f_{(zu)}^{(2,f)}     \rp^{\c_2}  }. \nonumber\\
\end{eqnarray}
For $\nu$ we first observe (modulo continuity)
\begin{eqnarray}\label{eq:pder11}
 \frac{d f_{(zc)}^{(2)} } {d\nu } & = &
 \mbox{sign}    \lp  \nu \rp
  2 \sinh \lp \left | \sqrt{\q_1^{(s)} -\q_2^{(s)} }\h_1^{(2)} + \sqrt{\q_2^{(s)} } \h_1^{(3)} \right |+ |\nu|  \rp,
\end{eqnarray}
and then
\begin{eqnarray}\label{eq:pder12}
  \frac{d  \bar{\psi}_{rd}(\p,\q,\c,\nu,\gamma_{sq},-1)} {d\nu  }
 &\rightarrow &
 \bar{\delta}
  - \frac{1}{\c_2}
  \mE_{ {\mathcal U}_3 } \frac{ \mE_{{\mathcal U}_2} \lp  \c_2 \lp f_{(zc)}^{(2)} \rp^{\c_2-1}   \frac{d f_{(zc)}^{(2)} } {d\nu }   \rp  }
    {  \mE_{{\mathcal U}_2} \lp  f_{(zc)}^{(2)}     \rp^{\c_2}  }.
    \end{eqnarray}
A combination of  (\ref{eq:pder2}), (\ref{eq:pder4}), (\ref{eq:pder7}), (\ref{eq:pder9}), (\ref{eq:pder10}), and (\ref{eq:pder12}) is sufficient to determine $\hat{\p}$, $\hat{\q}$, $\hat{\c}$, and $\hat{\nu}$. Together with $\hat{\gamma}_{sq}\rightarrow 0$ this then suffices to determine $\bar{\psi}_{rd}(\hat{\p},\hat{\q},\hat{\c},\hat{\nu},\hat{\gamma}_{sq},-1)$ as well. The obtained results for $\alpha=0.77$ and $\alpha=0.78$ are shown in Figures \ref{fig:fig0} and \ref{fig:fig1}. In Figure  \ref{fig:fig0} we also show the maximal LE value, $S_{max}$, as a useful upper bound and an indicator how large the considered clusters are compared to their maximal size.
\begin{figure}[h]
\centering
\centerline{\includegraphics[width=1.00\linewidth]{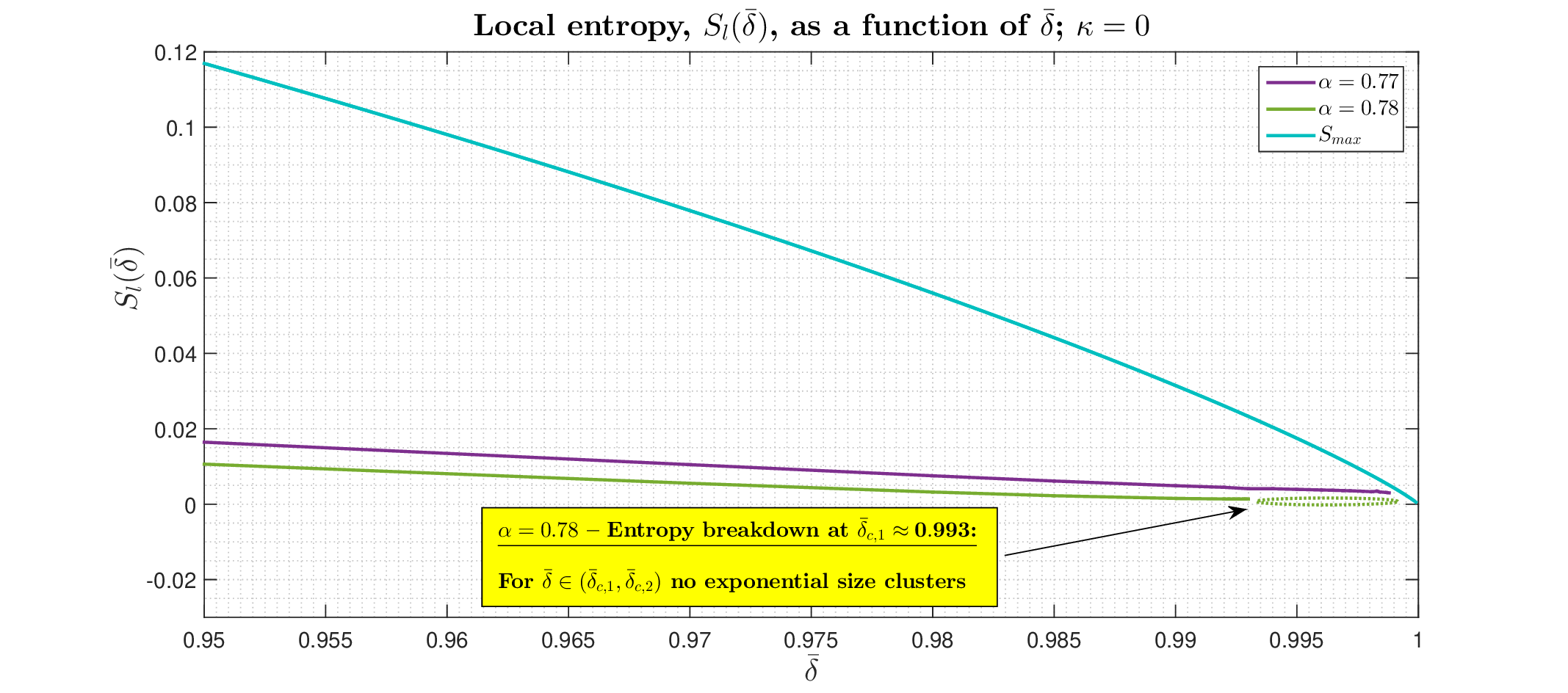}}
\caption{Local entropy breakdown}
\label{fig:fig0}
\end{figure}
In  Table \ref{tab:tab1},  the results from Figures \ref{fig:fig0} and \ref{fig:fig1}  are complemented with concrete values of all relevant parameters for a particular overlap value, $\bar{\delta}=0.99$ (we should also add that we repeated all of the above calculations relying on modulo-$\m$ concepts from  \cite{Stojnicnflldp25} and \cite{Stojnicsflldp25} and obtained exactly the same results as those in Figures \ref{fig:fig0} and  \ref{fig:fig1} and Table \ref{tab:tab1} which implies \emph{maximization} type of  $\c$ \emph{stationarity} and a remarkable agreement with the same type of discovery discussed in \cite{Stojnicbinperflrdt23,Stojnichopflrdt23,Stojnicnegsphflrdt23}).

\begin{table}[t]
\caption{$r$-sfl RDT parameters; ABP local entropy;  $\hat{\c}_1\rightarrow 1$; $\kappa=0$; $n,p,\beta\rightarrow\infty$}\vspace{.1in}
\centering
\def\arraystretch{1.2}
{\small \begin{tabular}{||l||c||c||c||c||c|c||c|c||c||c||}\hline\hline
 \hspace{-0in}$r$-sfl RDT                                             & $\alpha$   & $\bar{\delta}$   & $\hat{\gamma}_{sq}$        & $\hat{\nu}$    &  $\hat{\p}_2$ & $\hat{\p}_1$     & $\hat{\q}_2^{(s)}\rightarrow \hat{\q}_2\beta^2$  & $\hat{\q}_1^{(s)}\rightarrow \hat{\q}_1\beta^2$ &  $\hat{\c}_2$    & $S_l(\bar{\delta})$  \\ \hline\hline
   $2$-sfl RDT                                      & $0.77$            & $0.99$            & $0$    & $0.2258$   & $0.6539$ & $0.9814$ &  $0.2087$ & $0.7924$
 &  $4.86$   & \bl{$\mathbf{0.0049}$}  \\ \hline\hline
   $2$-sfl RDT                                      & $0.78$            & $0.99$            & $0$    & $0.0983$   & $0.6536$ & $0.9818$ &  $0.2529$ & $0.9544$
 &  $4.40$   & \bl{$\mathbf{0.0015}$}  \\ \hline\hline
  \end{tabular}}
\label{tab:tab1}
\end{table}

The most distinct feature concretely related to Figure \ref{fig:fig0} is the appearance of a local entropy breakdown  for $\alpha=0.78$. As  Figure \ref{fig:fig1} further shows, for overlap values $\bar{\delta}\geq \bar{\delta}_{c,1}\approx 0.993$, there is no $\nu\neq 0$ that would allow to enforce the overlap and existence of an exponentially large number of solutions at distance $d=\frac{1-\bar{\delta}}{2}$ from any other point. This immediately implies that there are no such clusters around any solutions either (the local entropy concept studied here is the worst case variant and as such is able to exclude existence of \emph{any} exponential clusters). This entropy breakdown trend seems to reverses only when  $\bar{\delta}\geq \bar{\delta}_{c,2}\approx 1$  (invisible in figures due to numerical instabilities of small values and being very close to 1).

It should also be noted from Figure \ref{fig:fig1} that $\nu>0$ remains in place for any $\bar{\delta}$ at $\alpha=0.77$ implying the absence of entropy breakdown as indicated in Figure \ref{fig:fig0}. This is in a perfect agreement with \cite{Bald16} where statistical physics replica methods were used to predicate that $(0.77,0.78)$ $\alpha$ window plays a key role in characterization of the worst case local entropy (it is also in a remarkable agreement with the fact that known practically feasible ABP algorithms typically have a breakdown precisely around $\alpha\sim 0.75-0.77$). Moreover, the results shown in Figures \ref{fig:fig0} and \ref{fig:fig1} seem incredibly similar to the ones in Figure 1 in \cite{Bald16}. Even more remarkable is that when $\delta q$ and $\delta \hat{q}$ in \cite{Bald16}'s (B37)-(B49) are close to zero, the curves are not only visually similar but also analytically completely identical.

One should keep in mind that all the evaluations discussed above are done on the second level of lifting. While results are likely to experience tiny refinements on higher levels we do not expect major qualitative changes.

\begin{figure}[h]
\centering
\centerline{\includegraphics[width=1.00\linewidth]{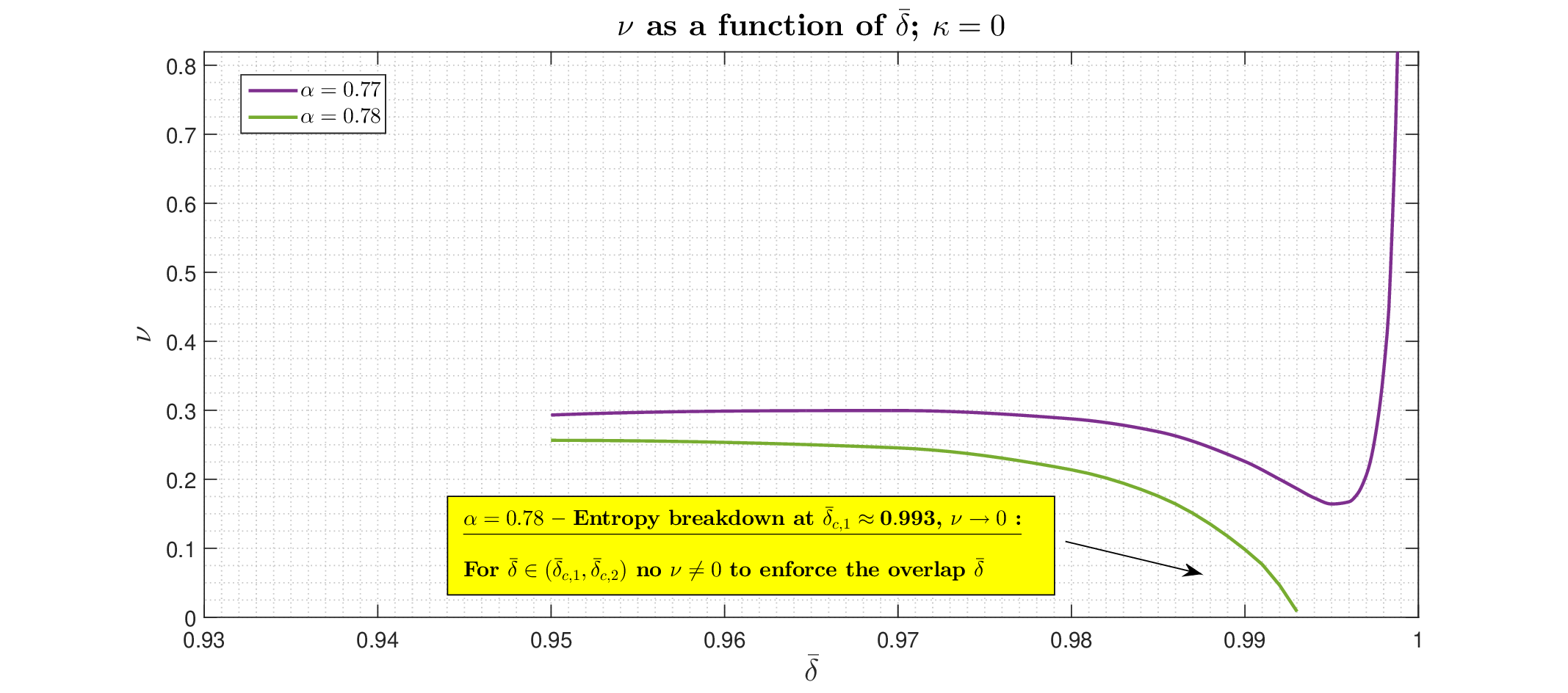}}
\caption{$\nu$ as a function of $\bar{\delta}$}
\label{fig:fig1}
\end{figure}

\section{Conclusion}
\label{sec:conc}

In a search for answers that would explain the ABP's algorithmic hardness and presumable existence of computational gaps  understanding solutions clusterings has been posited as of key importance \cite{Bald15,Bald16,AbbLiSly21a,Gamar21,GamKizPerXu22}. Two problem intrinsic features, overlap gap property (OGP) and local entropy (LE), as driving forces behind specific clustering organizations, are predicated to play essential role in studying computational gaps. Even when \emph{typical} problem solutions are isolated thereby disallowing the presence of large number of well connected clusters, there exist \emph{atypical} exponentially large rare clusters that efficient algorithms magically discover. \cite{Bald15,Bald16} go further and predict that LE's monotonicity or complete breakdown are directly related to the appearance of such atypicalities.

Invention of fully lifted random duality theory (fl RDT) \cite{Stojnicnflgscompyx23,Stojnicsflgscompyx23,Stojnicflrdt23} allowed studying random structures \emph{typical} features (for example, ABP's capacities and typical solutions entropies can easily be handled). A large deviation upgrade (sfl LD RDT) introduced in companion papers  \cite{Stojnicnflldp25,Stojnicsflldp25} allows one to move further and study the \emph{atypical} features as well. The programs developed in \cite{Stojnicnflldp25,Stojnicsflldp25} are utilized here to study ABP's LE as an example of such atypical feature. We establish a powerful generic framework and discover that already on the second level of full lifting the results are closely matching those obtained through replica methods. In particular, we obtain that LE breaks down for $\alpha$ in $(0.77,0.78)$ interval. This is both not that far away from ABP's capacity $\alpha_c=0.833$ and very close to the limits of current practical algorithms ($\alpha\sim 0.75-0.77$)  which indicates that LE might indeed be associated with the existence of ABP's computational gaps.

The introduced methodology is a generic concept and allows for various extensions beyond ABP. Many of them are related to random feasibility problems such as symmetric binary perceptrons (SBP)  positive/negative  spherical perceptrons, multi-layer (deep) neural nets, and \emph{strong/sectional} compressed sensing $\ell_1$ phase transitions. Many also relate to generic random structures such as  discrepancy minimization problems and those discussed in \cite{Stojnicnflgscompyx23,Stojnicsflgscompyx23,Stojnicflrdt23,Stojnichopflrdt23}. Obtaining concrete results for these extensions typically requires a bit of additional technical adjustment that is problem specific and will be discussed in separate papers.

\begin{singlespace}
\bibliographystyle{plain}
\bibliography{nflgscompyxRefs}
\end{singlespace}

\end{document}